\documentclass[10pt,twocolumn]{IEEEtran}
\usepackage{tikz}
\usepackage{enumitem}
\usepackage{graphicx}
\usepackage{amssymb}
\usepackage{xcolor}
\usetikzlibrary{positioning, decorations.pathreplacing}

\IEEEoverridecommandlockouts

\usepackage{amsmath, amssymb, cite}
\usepackage{amsthm,bm,bbm}
\usepackage{mathtools}

\usepackage[font=footnotesize,labelfont=bf]{caption}
\usepackage{amsthm,multirow,color,amsfonts}
\usepackage{tabulary}
\usepackage{subfigure}
\usepackage{graphicx}
\usepackage{setspace}
\usepackage{enumerate}

\usepackage[ruled]{algorithm2e}

\usepackage{comment}
\usepackage{pdfpages}
\usepackage{hyperref}
\usepackage{stfloats}
\usepackage{lscape, hhline}

\usepackage{etoolbox}
\makeatletter
\patchcmd{\@makecaption}
  {\scshape}
  {}
  {}
  {}
\makeatother

\makeatletter
\def\BState{\State\hskip-\ALG@thistlm}
\makeatother

\usepackage[utf8]{inputenc} 
\usepackage[T1]{fontenc}
\usepackage{url}
\usepackage{ifthen}

\usepackage[dvipsnames]{xcolor}

\newtheorem{theo}{Theorem}

\newtheorem{remark}{Remark}

\def\compactify{\itemsep=0pt \topsep=0pt \partopsep=0pt \parsep=0pt}
\let\latexusecounter=\usecounter

\ifx\du\undefined
  \newlength{\du}
\fi
\usepackage{siunitx}
\usepackage{tikz}
\usetikzlibrary{shapes.geometric}
\usepackage{verbatim}
\usepackage{color,soul}
\usetikzlibrary{spy,calc}
\hypersetup{breaklinks=true}
\usepackage{diagbox}

\begin{document}

\title{Mixture-of-Experts under Finite-Rate Gating: \\ Communication--Generalization Trade-offs}

\author{Ali Khalesi  and Mohammad Reza Deylam Salehi~
\thanks{A.~Khalesi is an \textit{Assistant Professor} at Institut Polytechnique des Sciences Avanc\'ees (IPSA) and LINCS Lab, Paris, France (ali.khalesi@ipsa.fr).
M.R.~Deylam Salehi is an IEEE member, Nice, France (reza.deylam@ieee.org).}
}

\maketitle
\begin{abstract}
Mixture-of-Experts (MoE) architectures decompose prediction tasks into specialized expert sub-networks selected by a gating mechanism. This letter adopts a communication-theoretic view of MoE gating, modeling the gate as a stochastic channel operating under a finite information rate. Within an information-theoretic learning framework, {we specialize a mutual-information generalization bound and develop a rate-distortion characterization $D(R_g)$ of finite-rate gating, where $R_g:=I(X; T)$, yielding (under a standard empirical rate-distortion optimality condition) $\mathbb{E}[R(W)] \le D(R_g)+\delta_m+\sqrt{(2/m)\, I(S; W)}$. }The analysis yields capacity-aware limits for communication-constrained MoE systems, and numerical simulations on synthetic multi-expert models empirically confirm the predicted trade-offs between gating rate, expressivity, and generalization.
\end{abstract}


\section{Introduction}
\label{sec:intro}

MoE models~\cite{jacobs1991adaptive,jordan1994hierarchical} combine multiple expert predictors via a gating network that assigns probabilistic weights or discrete routing decisions.
This modular structure enables specialization, scalability, and sparse activation in large architectures such as the Switch Transformer~\cite{fedus2022switch}, where only a small subset of experts is active per input. Despite their practical success, the theoretical understanding of MoE systems---particularly generalization under finite communication resources---remains limited.
Classical analyses~\cite{azran2004data} gave Rademacher-based bounds that additively depend on the gate complexity and the sum of expert complexities, scaling linearly with the number of experts. More recently, Akretche~\emph{et al.}~\cite{akretche2024tighter} introduced local differential privacy (LDP) regularization on the gate, obtaining tighter PAC-Bayesian and Rademacher bounds with logarithmic dependence on the number of experts. However, even in this setting, the gate is treated as a stochastic selector rather than an information-constrained communication process.

The proposed communication-generalization framework for MoE, beyond its theoretical significance in distributed learning, has a natural interpretation in aeronautical and aerospace communication systems~\cite{Toso2022_AeroDistributed, Li2023_UAVSurvey,11271575}. Modern aircraft, satellites, and unmanned aerial vehicles (UAVs) increasingly rely on distributed sensing and computation, where multiple onboard or remote modules act as \emph{experts} processing heterogeneous sensor streams under stringent bandwidth, latency, and reliability constraints.
The MoE gate parallels the routing logic in such systems---deciding which local estimator or control unit should communicate with the flight computer or ground segment. In this context, the mutual-information rate constraints derived here quantify the performance degradation (in estimation, navigation, or control) induced by limited communication capacity, extending \emph{data-rate theorems}~\cite{Nair2013_DataRate, Martins2016_CommControl} to learned, data-driven aerospace decision systems.

\textbf{Novelty and Contribution:}
From a communication-theoretic perspective, this work develops an explicit \emph{communication-generalization trade-off} for MoE models by reinterpreting the gating mechanism as a finite-capacity stochastic channel. We quantify the gating pathway by its mutual information rate $R_g=I(X; T)$, which limits how much information about the input can reach the expert bank and thus controls expressivity and statistical robustness. Building on this view, we introduce a rate-distortion formulation of the gating problem and show how the best achievable prediction risk at a given gating rate is characterized by a rate-distortion function $D(R_g)$, leading to the bound $\mathbb{E}[R(W)] \le D(R_g)+\delta_m+\sqrt{(2/m)\, I(S; W)}$, where $R_g$ enters through $D(R_g)$. We further consider the practically relevant case where gate decisions are conveyed over a physical link of capacity $C$; when the gate is trained/regularized to operate near this budget (so that $I(X;T)\approx C$), the bound specializes to $D(C)$, making the dependence on $C$ explicit, and connects MoE gating to classical notions such as channel capacity and data-rate limitations.

While information-theoretic generalization theory~\cite{xu2017information,bu2020tightening}, communication-limited learning~\cite{polyanskiy2022info_contraction,shamir2022communication}, and recent MoE risk bounds~\cite{akretche2024tighter} are well established, to the best of our knowledge, while several works interpret specialization and routing through information-theoretic lenses, we are not aware of prior MoE risk bounds in which the gate is modeled explicitly as a rate-limited channel and the achievable risk is expressed with the gating rate $I(X;T)$ appearing as an explicit design parameter through a rate--distortion function. Related information-theoretic perspectives on specialization and hierarchical decision systems have also been studied,
e.g., via mutual-information principles and online learning in hierarchical architectures~\cite{hihn2019information, hihn2023hierarchically, hihn2024online}.
Our contribution is to place MoE architectures within a unified rate-distortion and capacity framework, where the gating rate $I(X; T)$ acts as a system-level communication constraint (which, in practice, is upper-bounded by the available link capacity) shaping both expressivity and generalization, and where privacy or compression mechanisms can be interpreted as additional noisy-channel layers. Theoretical results are supported by numerical simulations on multi-expert MoE models and on a binary symmetric \emph{one-bit gating} scenario, which empirically illustrate the predicted trade-offs between gating rate, sample size, and generalization performance.

\textbf{Organization}:
Section~\ref{sec:moe-comm} formalizes MoE models as stochastic communication systems. Section~\ref{sec:info-bound} applies information-theoretic generalization bounds to MoE gating. Section~\ref{sec:rd-section} introduces a rate-distortion formulation of the gating mechanism and establishes a rate-distortion model combined with an information-theoretic generalization term. Section~\ref{sec:sims} presents numerical simulations that empirically validate the theoretical trade-offs, and Section~\ref{sec:conclusion} concludes with perspectives for communication-aware MoE design.

\section{MoE as a Communication System}
\label{sec:moe-comm}


\textbf{System Model:} Let $(X,Y)\sim \mathcal{D}$ denote a random input--output pair, where $X\in\mathcal X\subseteq\mathbb R^d$ is a feature vector and $Y\in\mathcal Y$ is the corresponding label. We consider a Mixture-of-Experts (MoE) model with $n$ experts indexed by $[n]\triangleq\{1,\dots,n\}$. Expert $g\in[n]$ is a predictor $h_g(\cdot;W_g):\mathcal X\to \widehat{\mathcal Y}$ parameterized by expert parameters $W_g$ (e.g., $\widehat{\mathcal Y}=\{0,1\}$ for classification or $\widehat{\mathcal Y}=\mathbb R$ for regression). We collect all expert parameters into the \emph{expert bank}
$W_{\text{exp}}\triangleq (W_1,\dots,W_n)$.

The gating network is parameterized by $W_{\text{gate}}$ and maps an input $x$ to a probability vector
$g_{W_{\text{gate}}}(x)=(g_{W_{\text{gate}},1}(x),\dots,g_{W_{\text{gate}},n}(x))\in \Delta_n$,
where $\Delta_n\triangleq\{p\in\mathbb R_+^n:\sum_{i=1}^n p_i=1\}$ and $g_{W_{\text{gate}},g}(x)$ denotes the probability of routing $x$ to expert $g$. Given $X=x$, the routing random variable $T\in[n]$ is sampled according to
$\mathbb P(T=g\mid X=x;W_{\text{gate}})=g_{W_{\text{gate}},g}(x)$ for $g\in[n]$. Conditioned on $T=g$, the model outputs
$\hat{Y}=h_T(X;W_T)$, where $W_T$ denotes the parameters of the selected expert, i.e., $W_T=W_g$ when $T=g$. We denote the overall (possibly randomized) model parameters by
$W\triangleq (W_{\text{gate}},W_{\text{exp}})$. Unless stated otherwise, routing is applied \emph{per sample} and is conditionally memoryless: given $W$ and inputs $\{X_j\}_{j=1}^m$, the routing variables $\{T_j\}_{j=1}^m$ are conditionally independent with $P(T_j\mid X_j,W)=P_{W_{\text{gate}}}(T\mid X_j)$ for each $j$. Under this modeling, the gate defines a discrete memoryless channel $X\to T$ with channel law $P_{W_{\text{gate}}}(T\mid X)$, whose output $T$ selects one expert from the bank (see Fig.~\ref{fig:system_model_ultra}). Note that we denote the data-generating distribution by $\mathcal D$ (not to be confused with the rate--distortion function $D(\cdot)$ in Section~\ref{sec:rd-section}). Also throughout, uppercase symbols denote random variables (e.g., $X,Y,T,\hat{Y},W$), while lowercase letters denote realizations (e.g., $x,y,t$) and deterministic indices (e.g., $g\in[n]$).

\textbf{Learning Setup and Risk}: Let $S=\{(x_j,y_j)\}_{j=1}^m\sim \mathcal{D}^m$ denote a training sample of $m$ i.i.d.\ input--output pairs, and let $W=({W_{\text{gate}}}, W_{\text{exp}})$ be the (possibly randomized) parameters produced by a learning algorithm trained on $S$. For a bounded loss function $\ell:\mathcal Y\times\mathcal Y\to[0,1]$, the population and empirical risks are defined as
\begin{align}
\label{eq-rate-11}
R(W)
&=\mathbb E_{(X,Y)\sim \mathcal{D}}
\mathbb E_{T\sim g_{W_{\text{gate}}}(X)}
\big[\ell(h_T(X;W_{T}),Y)\big],\\[1mm]
R_S(W)
&=\frac{1}{m}\sum_{j=1}^m \mathbb E_{T\sim g_{W_{\text{gate}}}(x_j)}
\big[\ell(h_T(x_j;W_T),y_j)\big].
\end{align}
Here $T\sim P_{W_{\text{gate}}}(\cdot\mid X)$ denotes the stochastic gate output.

\textbf{Communication-Constrained Learning Objective} From a communication-theoretic perspective, the gating mechanism constitutes the only pathway through which information about the input $X$ is conveyed to the expert bank. We quantify this pathway by the \emph{gating information rate}, i.e.,
$R_g \triangleq I(X;T)$, which measures how much information about $X$ can be transmitted through the
routing decision. In practice, $R_g$ may be limited either implicitly (e.g., by regularization or noise) or explicitly by a physical communication link of finite capacity. The objective of interest is to minimize the population risk subject to a constraint on the gating information rate:
\begin{align}
\min_W \; R(W)
\quad \text{subject to} \quad I(X;T) \le \bar{R}_g\ ,
\end{align}
or, equivalently, under an exogenous capacity constraint $R_g\le C$. This formulation parallels classical communication problems in which expected distortion is minimized subject to a rate constraint, with prediction loss playing the role of distortion\footnote{{Throughout, we use the term \emph{gating rate} for the achieved mutual information $R_g=I(X; T)$, and we reserve the term \emph{capacity} for an external upper bound on this rate (e.g., the Shannon capacity of a physical link).}} {This is the standard communication objective of minimizing an application-level distortion (here, prediction loss) subject to an information-rate constraint on the message $T$. Section~\ref{sec:rd-section} formalizes this via a rate--distortion function.}

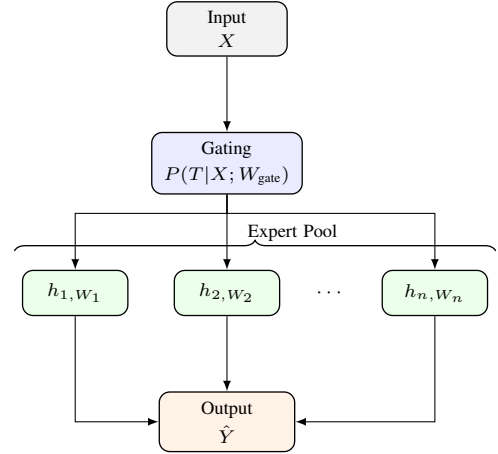
\begin{figure}[!t]
\centering
\scalebox{1}{
\begin{tikzpicture}[>=latex, font=\scriptsize, node distance=0.9cm, line width=0.5pt]

\node[draw, rounded corners, fill=gray!10, minimum width=1.6cm, minimum height=0.6cm] 
    (input) {\shortstack{Input\\$X$}};
\node[draw, rounded corners, fill=blue!8, below=1.0cm of input, minimum width=2.0cm, minimum height=0.6cm]
    (gating) {\shortstack{Gating\\$P(T|X;W_{\text{gate}})$}};

\node[draw, rounded corners, fill=green!8, below=1.0cm of gating, minimum width=1.4cm, minimum height=0.6cm]
    (expert2) {$h_{2,W_2}$};
\node[draw, rounded corners, fill=green!8, left=0.6cm of expert2, minimum width=1.4cm, minimum height=0.6cm]
    (expert1) {$h_{1,W_1}$};

\node[right=0.35cm of expert2] (dots) {$\cdots$};
\node[draw, rounded corners, fill=green!8, right=0.35cm of dots, minimum width=1.4cm, minimum height=0.6cm]
    (expertN) {$h_{n,W_n}$};

\node[draw, rounded corners, fill=orange!10, below=1.0cm of expert2, minimum width=1.8cm, minimum height=0.6cm]
    (output) {\shortstack{Output\\$\hat{Y}$}};

\draw[->, thin] (input) -- (gating);
\draw[->, thin] (gating.south) -- ++(0,-0.25) -| (expert1.north);
\draw[->, thin] (gating.south) -- (expert2.north);
\draw[->, thin] (gating.south) -- ++(0,-0.25) -| (expertN.north);
\draw[->, thin] (expert1.south) |- (output.west);
\draw[->, thin] (expert2.south) -- (output);
\draw[->, thin] (expertN.south) |- (output.east);

\draw [decorate,decoration={brace,amplitude=4pt}]
      ([xshift=-0.1cm,yshift=0.25cm]expert1.north west) -- 
      ([xshift=0.1cm,yshift=0.25cm]expertN.north east)
      coordinate[midway] (bracepos);

\node[xshift=0.5cm, above=0.01cm of bracepos] {\scriptsize Expert Pool};

\end{tikzpicture}}
\caption{MoE system viewed as a finite-rate stochastic communication link. The input feature vector $X$ is processed by a gating module implementing a channel $P(T \mid X; W_{\text{gate}})$ that maps $X$ to a routing index $T\in[n]$ under an information-rate constraint. The index selects an expert ${h_g(\cdot;W_g)}_{g=1}^n$, yielding the prediction $\hat{Y}=h_T(X;W_T)$. Here $I(X;T)$ is the effective communication rate to the expert bank, so the gate acts as a constrained link that limits how much information about $X$ reaches the experts, shaping expressivity and generalization.}

\label{fig:system_model_ultra}
\end{figure}
\section{Information-Theoretic Bound and Proof}
\label{sec:info-bound}

We adopt the framework of information-theoretic learning theory~\cite{xu2017information,bu2020tightening}, {our contribution is the communication interpretation of the gating randomness and its integration with the rate-distortion. In particular,  The following is a direct specialization of the mutual-information generalization bound of Xu and Raginsky to the MoE prediction rule with internal stochastic gating.} Let $I(S; W)$ denote the mutual information between the training sample and the learned parameters. This quantity measures how much the learned model depends on the data and therefore controls its expected generalization gap.

\begin{theo}[{Xu--Raginsky bound specialized to MoE}]
\label{th:itmoe}
Assume $\ell \in [0,1]$. Let $W$ be the (possibly randomized) parameters produced by a learning algorithm trained on sample $S=\{(x_j,y_j)\}_{j=1}^m\sim \mathcal{D}^m$. Then, for any sample size $m \ge 1$, the expected generalization gap satisfies
\begin{align}
\label{eq:itmoe}
\big|\mathbb{E}[R(W)] - \mathbb{E}[R_S(W)]\big|
  \le  
\sqrt{\frac{2}{m}\, I(S;W)} ,
\end{align}
where $I(S; W)$ in (\ref{eq:itmoe}) is the mutual information between the training sample and the learned parameters\footnote{Unless stated otherwise, all mutual information and entropy quantities are measured in nats, i.e., logarithms are natural ($\ln$). References to \emph{bits} are purely informal; using base-2 logarithms would simply rescale these quantities by a factor of $\ln 2$.}.
\end{theo}
{The proof is given in Appendix~\ref{proofTheorem1}.}

\begin{remark}[Communication interpretation]
\label{rem-com-int}
The bound in Theorem~\ref{th:itmoe} quantifies the \emph{estimation} or learning-algorithm component of generalization through $I(S; W)$.
Separately, the gating mechanism defines a communication channel $X\to T$ with rate $I(X; T)$.
Constraining this rate---e.g., via local differential privacy (LDP), where each input is first randomized by a local privacy mechanism, or via a rate-distortion objective---limits how much information the gate can transmit from the input to the experts, thereby controlling the model capacity. Thus, generalization is governed by $I(S;W)$, while expressivity and communication dependence are governed by $I(X; T)$. The novelty of our framework is that, \emph{although the mutual-information bound itself follows classical results, interpreting the gating mechanism as a finite-rate communication channel and explicitly linking its communication rate $I(X; T)$ to model expressivity and generalization has not appeared in prior analyses of MoE.}
\end{remark}

\begin{remark}[{On estimating $I(X;T)$ and $I(S;W)$ in practice}]
{
For discrete routing $T\in[n]$ with gate probabilities $g_{W_{\text{gate}}}$, the gating rate admits the computable form
\begin{align*}
I(X;T)=\mathbb E\Big[\log \frac{g_{W_{\text{gate}},T}(X)}{\pi_T}\Big], \pi_g\triangleq \mathbb E[g_{W_{\text{gate}},g}(X)],\ g\in[n]
\end{align*}
which can be approximated by empirical averages over a held-out dataset or minibatches.
For high-dimensional learned parameters, $I(S;W)$ is generally intractable; in practice one may upper bound or proxy it via PAC-Bayesian/compression-based estimates or variational mutual-information bounds, using it as a \emph{design heuristic} rather than an exactly computed quantity \cite{belghazi2018mine}.
}
\end{remark}

\section{Rate-Distortion Formulation of Gating}
\label{sec:rd-section}
{The gating mechanism can also be interpreted through a \emph{rate-distortion} lens. It must encode each input $X$ into a discrete message of limited communication
rate that still yields low prediction loss. In the MoE setting, the rate-distortion index $T$ corresponds to the routing variable $T$ produced by the gating mechanism.}

\subsection{Rate-Distortion Optimization}

The gating design problem can be posed as:
\begin{align}
\label{eq:RDproblem}
&\min_{P(T|X)} \mathbb{E}\!\big[\ell(h_T(X;W_T),Y)\big],
\\ &\text{s.t. } I(X;T)\le R_g, \nonumber
\end{align}
where the expectation is over an independent $(X,Y)\sim \mathcal{D}$ and $T|X\sim P(\cdot|X)$. {Classical Shannon rate--distortion theory is an asymptotic statement: it characterizes the minimum achievable expected distortion when encoding long blocks of i.i.d.\ source samples with blocklength $L\to\infty$ under an average rate constraint. 
In our setting, we assume $(X_j,Y_j)$ are i.i.d.\ and the gate is memoryless across samples conditioned on $(X,W)$; thus $D(R_g)$ is interpreted as the \emph{single-letter} distortion--rate function induced by the realized expert bank $W_{\text{exp}}$ through the per-sample selection $T$ and the incurred loss $\ell(h_T(X;W_T),Y)$; the term $\delta_m$ captures finite-sample/finite-optimization effects when training a rate-regularized gate from data.}
The above expresses the fundamental trade-off between communication rate and predictive performance. Equivalently, by introducing a Lagrange multiplier $\lambda>0$, we can minimize
\begin{align}
\label{eq:lag}
\mathcal{L}(P(T|X))=\mathbb{E}\!\big[\ell(h_T(X;W_T),Y)\big] + \lambda\,I(X;T).
\end{align}
{The constrained problem (\ref{eq:RDproblem}) and its Lagrangian form (\ref{eq:lag}) can be solved by classical rate--distortion iterations (e.g., Blahut-Arimoto) in settings where the distributions are tractable~\cite{blahut1972hypothesis}.}

\subsection{Rate-Distortion Function and Generalization Bound}

We define the rate--distortion function
\begin{align}
\label{eq:DR}
D(R_g)\triangleq \inf_{I(X;T)\le R_g}\, \mathbb{E}\!\big[\ell(h_T(X;W_T),Y)\big].
\end{align}
When $W_{\mathrm{exp}}$ is learned from $S$, $D_{W_{\mathrm{exp}}}(R_g)$ is a random quantity; throughout, $D(R_g)$ should be interpreted conditionally on the realized bank $W_{\mathrm{exp}}$.
As in classical rate--distortion theory, $D(R)$ is non-increasing in $R$. Throughout, we view $D(R_g)$ as induced by a fixed expert bank $W_{\text{exp}}$ (i.e., $D(R_g)\equiv D_{W_{\text{exp}}}(R_g)$), and the infimum ranges over all conditional laws $P(T|X)$ satisfying the rate constraint. Thus, $D(R_g)$ is the best achievable prediction loss under gating rate $R_g$. The following theorem combines rate--distortion and generalization principles.

\begin{theo}[Rate-Distortion-Generalization Bound]
\label{th:rdmoe}
Let $W=({W_{\text{gate}}}, W_{\text{exp}})$ denote the parameters produced by a learning algorithm trained on a sample $S\sim {\mathcal{D}^m}$, and let $T$ denote the corresponding (possibly stochastic) gating output. Define the effective gating rate as $R_g:= I(X; T)$. Assume moreover that the expected empirical risk satisfies
\begin{align}
\label{eq-expext-r-s}
\mathbb{E}[R_S(W)] \le D(R_g) + \delta_m\ ,
\end{align}
for some nonnegative error term $\delta_m \ge 0$ (capturing optimization and empirical-to-population mismatch for the rate-distortion objective at rate $R_g$). Such a condition is satisfied, for example, when the gating parameters ${W_{\text{gate}}}$ are obtained (for the realized expert bank $W_{\text{exp}}$) by approximately minimizing the empirical rate-distortion Lagrangian
\begin{align}
\label{eq:empRD}
\widehat{\mathcal L}_S(W_{\text{gate}};W_{\text{exp}})
\triangleq R_S(W_{\text{gate}},W_{\text{exp}})+\lambda\, I(X;T),
\end{align}
for some $\lambda>0$, or equivalently by solving the constrained empirical problem in \eqref{eq:RDproblem}. In this case, $\delta_m$ captures the optimization sub-optimality and the empirical-to-population mismatch of the rate-distortion objective at rate $R_g$. Then the expected population risk of the learned MoE satisfies
\begin{align}
\label{eq:RDgen}
\mathbb{E}[R(W)]\le D(R_g)+\delta_m+\sqrt{\frac{2}{m}\,I(S;W)}\ .
\end{align}
\end{theo}
{The proof is given in Appendix~\ref{proofofTheorem2}. 
It combines \eqref{eq-expext-r-s} with Theorem~1.}

\begin{remark}[Generalization trade-off in MoE under local privacy]
Theorem~\ref{th:itmoe} provides an explicit \emph{communication-generalization frontier} for MoE systems. Suppose that the gating mechanism is designed through an $\epsilon$-LDP channel, meaning that each input $X$ is passed through a randomized mechanism $Q(\cdot|X)$ satisfying the per-sample privacy condition $Q(z\mid x) \le e^{\epsilon} Q(z\mid x')$ for all $x,x',z$. For such locally private channels, standard information bounds (e.g.,~\cite{duchi2013local}) imply that there exists a constant $\zeta>0$ (in nats) such that $I(X;T)\le \zeta,\epsilon^2$. Therefore, $\epsilon$-LDP directly limits the gating information rate $R_g=I(X;T)$ and hence constrains the achievable rate-distortion performance through $D(R_g)$. In particular, if $R_g \le \zeta\epsilon^2$, then by Theorem~\ref{th:rdmoe} we obtain $\mathbb{E}[R(W)] \le D(\zeta\epsilon^2)+\delta_m+\sqrt{\frac{2}{m}I(S;W)}$. Any further reduction of $I(S;W)$ due to privacy mechanisms depends on the learning algorithm and is not asserted here.
\end{remark}

\begin{remark}[Capacity-limited MoE gating]
\label{eq-rem-cap-moe}
Suppose that the gating decisions $T$ are conveyed over a physical communication link with per-sample Shannon capacity $C$ nats, so that the achieved gating rate
\begin{align*}
    R_g \triangleq I(X; T) \le C\ .
\end{align*}
Since $D(\cdot)$ is non-increasing, $R_g\le C$ implies
\begin{align*}
D(R_g)\ge D(C).
\end{align*}
Therefore, Theorem~\ref{th:rdmoe} holds as stated with the achieved rate $R_g$:
\begin{align}
\mathbb{E}[R(W)]   \le   D(R_g) +\delta_m+ \sqrt{\frac{2}{m}I(S;W)}\ , \:\: R_g\le C\  .
\end{align}
{In particular, $R_g=I(X;T)$ is the fundamental bottleneck: any privacy-preserving randomization that reduces mutual information necessarily increases the minimum achievable distortion through $D(R_g)$.}
Moreover, under a near-saturation design in which $I(X;T) \approx C$ (and \eqref{eq-expext-r-s} holds at rate of $C$), we have $D(R_g)\approx D(C)$, so the bound effectively specializes to
\begin{align}
\mathbb{E}[R(W)]   \lesssim   D(C) +\delta_m+ \sqrt{\frac{2}{m}I(S;W)}\ .
\end{align}
\end{remark}

\section{Numerical Simulations}
\label{sec:sims}
We next illustrate the applicability of Theorems~\ref{th:itmoe} and~\ref{th:rdmoe} on two synthetic experiments. In the first experiment (Fig.~\ref{fig:thm1-sim}), we consider a binary classification task generated by a {true} MoE model with real-valued Gaussian features. The true model has $d=3$ input dimensions and $10$ experts. A bank of $30$ candidate MoE models is drawn i.i.d.\ from the same prior distribution. For each dataset of size $m=5$, we compute the empirical $0$--
$1$ risk of every candidate model and use an $\alpha$-mixture learning algorithm: with probability $1-\alpha$ the learner samples a model uniformly at random, and with probability $\alpha$ it selects the empirical-risk minimizer. By sweeping $\alpha$ from $0$ to $1$, we obtain a family of learning algorithms with increasing dependence on the sample and hence increasing mutual information $I(S; W)$, estimated via $I(S; W)=H(W)-H(W|S)$ from the posterior $q(W|S)$. For each $\alpha$ we estimate the generalization gap $\lvert \mathbb{E}[R(W)-R_S(W)]\rvert$ by averaging over $100$ datasets and a separate test set of size $1000$.

Fig.~\ref{fig:thm1-sim} plots the empirical generalization gap as a function of the information term $\sqrt{2I(S;W)/m}$, together with the theoretical line $y=\sqrt{2I(S;W)/m}$. The points form an increasing curve that stays strictly below the diagonal, showing that the mutual-information bound of Theorem~\ref{th:itmoe} is respected and captures the growth of the generalization gap as the learner becomes more data-dependent.

\begin{figure}[t]
    \centering
    \includegraphics[width=0.95\linewidth]{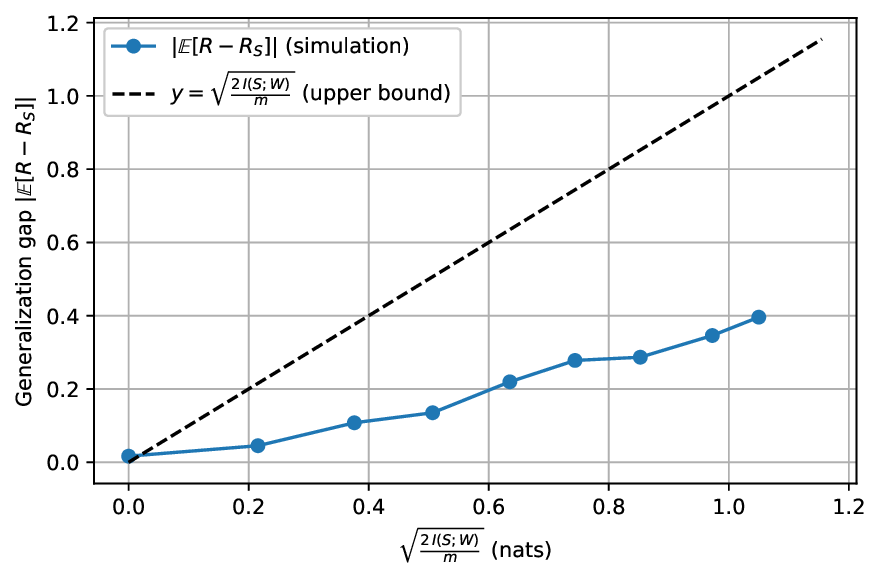}
    \caption{MoE simulation for Theorem~\ref{th:itmoe}. The empirical generalization gap $\lvert\mathbb{E}[R-R_S]\rvert$ increases with the information term $\sqrt{2 I(S; W)/m}$ and remains below the theoretical upper bound, illustrating the info-generalization trade-off at the algorithm level.}
    \label{fig:thm1-sim}
\end{figure}

In the second experiment (Fig.~\ref{fig:thm2-sim}), we use a binary symmetric channel (BSC) setting that makes the rate-distortion structure explicit. We let $X\sim\mathrm{Bernoulli}(1/2)$ and $Y=X\oplus Z$, where $Z\sim\mathrm{Bernoulli}(p_{\text{true}})$ models noise. A finite candidate set $\{p_r\}$ of crossover probabilities is fixed, and for each dataset of size $m=200$, we compute the log-likelihood of the observed sample under every candidate BSC$(p_r)$, form a tempered posterior over indices, and sample a single channel index $W$. The resulting model has crossover probability $p_W$, so its population risk is $R(W)=p_W$, and its gating rate is $R_g= I(X; T)=\ln 2-h(p_W)$ in nats, where $h(\cdot)$ denotes  the binary entropy function (in nats). This toy channel can be viewed as a one-bit gating scenario, where the crossover probability $p_W$ measures how accurately the gate routes $X$ to the appropriate expert, and the rate $R_g=\ln 2 - h(p_W)$ represents the effective communication rate of this binary gating path. For each $p_{\text{true}}$ in a grid, we run $400$ Monte Carlo datasets, estimate $I(S; W)$ via entropies of the posterior, and evaluate the bound from Theorem~\ref{th:rdmoe} as
\begin{align}
\label{eq-eval-bound}
D(R_g) + \sqrt{({2}/{m})I(S;W)}\ ,
\end{align}
where $D(R_g)$ is the BSC rate-distortion function in nats \footnote{Note that, in this BSC setting, $D(R_g)$ is available in closed form and the rate term is evaluated analytically, so $\delta_m$ is negligible in the plotted bound.}.

Fig.~\ref{fig:thm2-sim} shows the analytical rate-distortion curve $D(R_g)$, the empirical mean population risks $\mathbb{E}[R(W)]$ obtained from the simulation, and the corresponding upper-bound values of Theorem~\ref{th:rdmoe}. The empirical risks track $D(R_g)$ closely across a wide range of gating rates, while all simulated points lie comfortably below the theoretical curve in (\ref{eq-eval-bound}), 
 thus illustrating the predicted rate-distortion-generalization trade-off. {To complement the synthetic simulations, we also provide in the Supplementary Material a small MNIST experiment based on a finite bank of pretrained CNN experts and a discrete selection rule, which enables empirical estimation of $I(S;W)$ and verification of the term $\sqrt{2I(S;W)/m}$.}

\begin{figure}[t]
    \centering
    \includegraphics[width=0.95\linewidth]{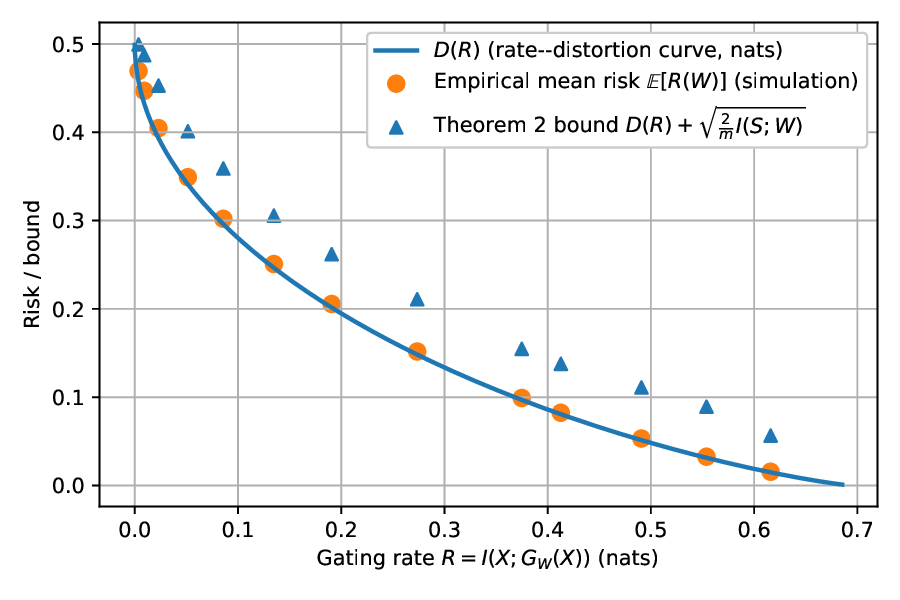}
    \caption{BSC rate-distortion-generalization experiment for Theorem~\ref{th:rdmoe}. The empirical mean population risk $\mathbb{E}[R(W)]$ (dots) closely follows the rate-distortion curve $D(R_g)$, while the bound in (\ref{eq-eval-bound}) 
    (triangles) remains safely above, confirming the theoretical trade-off between gating rate and prediction accuracy.}
    \label{fig:thm2-sim}
\end{figure}

\section{Conclusion}
\label{sec:conclusion}

This letter established an explicit information-theoretic connection between the communication rate of the gating mechanism and the generalization performance of Mixture-of-Experts models. By casting the gate as a finite-capacity stochastic channel, we derived mutual-information and rate-distortion-generalization bounds in which the gating rate $I(X; T)$ appears as an operational capacity measure shaping expressivity and statistical robustness.

From a practical standpoint, the analysis applies to communication-constrained learning settings such as federated MoE, edge or split inference, coded computing, and distributed aerospace systems, where expert routing must operate under strict bandwidth or reliability constraints~\cite{Toso2022_AeroDistributed, Li2023_UAVSurvey}.
In such settings, the gating rate $I(X; T)$ acts as an operational communication rate (upper-bounded by the available link capacity) that governs statistical efficiency and model expressivity; in aerospace and autonomous platforms, it also captures trade-offs between communication rate, inference accuracy, and control responsiveness, extending classical data-rate theorems~\cite{Nair2013_DataRate, Martins2016_CommControl} to modern learned decision systems. Future work includes extending these results to deeper and hierarchical MoE architectures, developing training procedures that explicitly regularize $I(X; T)$, and validating the theory on large-scale, bandwidth-limited deployments. While our analysis is single-letter and focuses on memoryless per-sample routing, extending it to deep/hierarchical MoE and to finite-blocklength routing over noisy physical links is nontrivial; in such regimes, $I(X;T)$ and $D(R_g)$ should be interpreted as design proxies and may be approximated via variational bounds or learned compression objectives.

\bibliographystyle{IEEEtran}
\bibliography{ref}

\begin{thebibliography}{10}
\providecommand{\url}[1]{#1}
\csname url@samestyle\endcsname
\providecommand{\newblock}{\relax}
\providecommand{\bibinfo}[2]{#2}
\providecommand{\BIBentrySTDinterwordspacing}{\spaceskip=0pt\relax}
\providecommand{\BIBentryALTinterwordstretchfactor}{4}
\providecommand{\BIBentryALTinterwordspacing}{\spaceskip=\fontdimen2\font plus
\BIBentryALTinterwordstretchfactor\fontdimen3\font minus \fontdimen4\font\relax}
\providecommand{\BIBforeignlanguage}[2]{{%
\expandafter\ifx\csname l@#1\endcsname\relax
\typeout{** WARNING: IEEEtran.bst: No hyphenation pattern has been}%
\typeout{** loaded for the language `#1'. Using the pattern for}%
\typeout{** the default language instead.}%
\else
\language=\csname l@#1\endcsname
\fi
#2}}
\providecommand{\BIBdecl}{\relax}
\BIBdecl

\bibitem{jacobs1991adaptive}
R.~A. Jacobs \emph{et~al.}, ``Adaptive mixtures of local experts,'' \emph{Neural computation}, vol.~3, no.~1, pp. 79--87, 1991.

\bibitem{jordan1994hierarchical}
M.~I. Jordan and R.~A. Jacobs, ``Hierarchical mixtures of experts and the em algorithm,'' \emph{Neural computation}, vol.~6, no.~2, pp. 181--214, 1994.

\bibitem{fedus2022switch}
W.~Fedus \emph{et~al.}, ``Switch transformers: Scaling to trillion parameter models with simple and efficient sparsity,'' \emph{J. M.L. Research}, vol.~23, no. 120, pp. 1--39, 2022.

\bibitem{azran2004data}
A.~Azran and R.~Meir, ``Data dependent risk bounds for hierarchical mixture of experts classifiers,'' in \emph{International Conference on Computational Learning Theory}.\hskip 1em plus 0.5em minus 0.4em\relax Springer, 2004, pp. 427--441.

\bibitem{akretche2024tighter}
W.~Akretche, F.~LeBlanc, and M.~Marchand, ``Tighter risk bounds for mixtures of experts,'' \emph{arXiv preprint arXiv:2410.10397}, 2024.

\bibitem{Toso2022_AeroDistributed}
F.~Toso and J.~P. How, ``Distributed estimation and control for aerospace systems: A survey of challenges and opportunities,'' \emph{IEEE Trans. on Aerospace and Electronic Systems}, vol.~58, no.~6, pp. 4879--4902, 2022.

\bibitem{Li2023_UAVSurvey}
Y.~Li \emph{et~al.}, ``A survey on intelligent unmanned aerial vehicle communications: A reinforcement learning perspective,'' \emph{IEEE Com. Surv \& Tutorials}, vol.~25, no.~1, pp. 473--510, 2023.

\bibitem{11271575}
A.~Khalesi and M.~R. Deylam~Salehi, ``Typical solutions of multi-user linearly-decomposable distributed computing,'' \emph{IEEE Networking Letters}, vol.~8, pp. 10--13, 2026.

\bibitem{Nair2013_DataRate}
G.~N. Nair and R.~J. Evans, ``Stabilizability of stochastic linear systems with finite feedback data rates,'' \emph{SIAM Journal on Control and Optimization}, vol.~43, no.~2, pp. 413--436, 2013.

\bibitem{Martins2016_CommControl}
N.~Martins and M.~Dahleh, ``Fundamental limitations of performance in the presence of limited communication,'' \emph{IEEE Trans. on Auto. Control}, vol.~61, no.~12, pp. 4111--4126, 2016.

\bibitem{xu2017information}
A.~Xu and M.~Raginsky, ``Information-theoretic analysis of generalization capability of learning algorithms,'' \emph{Advances in neural information processing systems}, vol.~30, 2017.

\bibitem{bu2020tightening}
Y.~Bu, S.~Zou, and V.~V. Veeravalli, ``Tightening mutual information-based bounds on generalization error,'' \emph{IEEE J. Sel. Areas info.}, vol.~1, no.~1, pp. 121--130, Apr. 2020.

\bibitem{polyanskiy2022info_contraction}
Y.~Polyanskiy and Y.~Wu, ``Strong data-processing inequalities and information contraction,'' \emph{Found. \& Trends in Comm. \& Info. Theory}, 2022.

\bibitem{shamir2022communication}
O.~Shamir, ``Communication-constrained learning: Fundamental limits and algorithms,'' in \emph{Conf. Learn. Theory}, 2022.

\bibitem{hihn2019information}
H.~Hihn, S.~Gottwald, and D.~A. Braun, ``An information-theoretic on-line learning principle for specialization in hierarchical decision-making systems,'' in \emph{IEEE Conf. Dec. and control (CDC)}.\hskip 1em plus 0.5em minus 0.4em\relax IEEE, 2019, pp. 3677--3684.

\bibitem{hihn2023hierarchically}
H.~Hihn and D.~A. Braun, ``Hierarchically structured task-agnostic continual learning,'' \emph{Machine Learning}, vol. 112, no.~2, pp. 655--686, 2023.

\bibitem{hihn2024online}
------, ``Online continual learning through unsupervised mutual information maximization,'' \emph{Neurocomputing}, vol. 578, p. 127422, 2024.

\bibitem{belghazi2018mine}
M.~I. Belghazi \emph{et~al.}, ``Mutual information neural estimation,'' in \emph{Proc. Int. Conf. Machine Learning (ICML)}, 2018.

\bibitem{blahut1972hypothesis}
R.~E. Blahut, \emph{An hypothesis testing approach to information theory}.\hskip 1em plus 0.5em minus 0.4em\relax Cornell University, 1972.

\bibitem{duchi2013local}
J.~C. Duchi \emph{et~al.}, ``Local privacy and statistical minimax rates,'' in \emph{Found. Comput. Science (FOCS)}.\hskip 1em plus 0.5em minus 0.4em\relax IEEE, 2013, pp. 429--438.

\end{thebibliography}
\appendices
\section{Proof of Theorem~1}\label{proofTheorem1}

\begin{proof}\label{proof-th:itmoe}
The training sample is $S=\{(X_j,Y_j)\}_{j=1}^m\sim \mathcal D^m$ (i.i.d.). A (possibly randomized)
learning algorithm takes $S$ as input and outputs model parameters
\begin{align*}
W =(W_{\text{gate}},W_{\text{exp}}),
\end{align*}
hence $W$ is a random variable whose law is induced by $(S,$ algorithmic randomness$)$.
To evaluate the population risk, we additionally draw an \emph{independent} test example
$(X,Y)\sim\mathcal D$ that is independent of $S$ and independent of the algorithmic randomness.
Given $W$ and the test input $X$, the gate samples a routing index
\begin{align*}
  T\sim g_{W_{\text{gate}}}(X)\quad\text{(i.e., } \mathbb P(T=g\mid X,W)=g_{W_{\text{gate}},g}(X)\text{)}.  
\end{align*}
All information quantities below are with respect to this joint distribution of
$$(S,W,X,Y,T).$$

Define the random loss on the test point as
\begin{align}
\label{eq:test0loss}
L \triangleq \ell\!\big(h_T(X;W_T),\,Y\big)\in[0,1].
\end{align}
The boundedness $L\in[0,1]$ holds by assumption on $\ell$.

By the definition of the MoE population risk,
\[
R(W)=\mathbb E\!\left[L\mid W\right],
\]
where the conditional expectation is over the independent test draw $(X,Y)\sim\mathcal D$ and over the
internal gate randomness $T\sim g_{W_{\text{gate}}}(X)$.
Similarly, the empirical risk is
\[
R_S(W)=\frac{1}{m}\sum_{j=1}^m \mathbb E\!\left[\ell\!\big(h_{T_j}(X_j;W_{T_j}),Y_j\big)\,\Big|\,S,W\right],
\]
where $T_j\sim g_{W_{\text{gate}}}(X_j)$ are the per-sample gate variables (conditionally i.i.d.\ given
$(S,W)$ under the memoryless-gating modeling). Taking total expectation over $(S,W)$ yields the
quantities appearing in the theorem statement:
$\mathbb E[R(W)]$ and $\mathbb E[R_S(W)]$.

We claim that, conditioned on $W$, the loss $L$ is independent of the training sample $S$:
\begin{equation}
\label{eq:condind}
P(L\mid W,S)=P(L\mid W).
\end{equation}
Indeed, given $W$,
\begin{itemize}
\item the test pair $(X,Y)$ is drawn from $\mathcal D$ independently of $S$;
\item the routing variable $T$ is sampled from $g_{W_{\text{gate}}}(X)$ using internal randomness that
is also independent of $S$ (once $W$ is fixed).
\end{itemize}
Therefore, the conditional law of $L$ depends on $S$ only through $W$, which is exactly
\eqref{eq:condind}. This is equivalent to the Markov chain
\begin{align}
\label{eq-markov}
  S \rightarrow W \rightarrow L.  
\end{align}
By the data-processing inequality, equation (\ref{eq-markov}) implies
\begin{align}
\label{eq:dpi}
I(S;L)\le I(S;W).
\end{align}

For bounded losses in $[0,1]$, Xu and Raginsky show (via their decoupling lemma) that
\begin{align}
\label{eq:xu-bound}
\big|\mathbb E[R(W)]-\mathbb E[R_S(W)]\big|
\le \sqrt{\frac{2}{m}\,I(S;L)}.
\end{align}
(Here $L$ is the single-sample test loss defined in \eqref{eq:test0loss}; see
\cite[Lemma~1 and Theorem~1]{xu2017information}.)

Combining \eqref{eq:xu-bound} with \eqref{eq:dpi} yields
\[
\big|\mathbb E[R(W)]-\mathbb E[R_S(W)]\big|
\le \sqrt{\frac{2}{m}\,I(S;W)},
\]
which is exactly \eqref{eq:itmoe}. \qedhere
\end{proof}

\section{Proof of Theorem~2}\label{proofofTheorem2}

\begin{proof}
For any two real numbers $a,b$, we have $a=b+(a-b)\le b+|a-b|$. Applying this with
$a=\mathbb E[R(W)]$ and $b=\mathbb E[R_S(W)]$ gives
\begin{align}
\label{eq:basic-decomp}
\mathbb E[R(W)]
\le \mathbb E[R_S(W)]
+ \big|\mathbb E[R(W)]-\mathbb E[R_S(W)]\big|.
\end{align}

By assumption \eqref{eq-expext-r-s}, the learned parameters satisfy
\begin{align}
\label{eq:assump}
\mathbb E[R_S(W)] \le D(R_g)+\delta_m,
\end{align}
where $R_g\triangleq I(X;T)$ is the \emph{achieved} gating rate induced by the learned gate and the data
distribution.
Substituting \eqref{eq:assump} into \eqref{eq:basic-decomp} yields
\begin{align}
\label{eq:reduce-to-gap}
\mathbb E[R(W)]
\le D(R_g)+\delta_m
+ \big|\mathbb E[R(W)]-\mathbb E[R_S(W)]\big|.
\end{align}

Define the instantaneous test loss on an independent test draw $(X,Y)\sim\mathcal D$ and internal gate
sample $T\sim g_{W_{\text{gate}}}(X)$ as
\begin{align}
    L \triangleq \ell\!\big(h_T(X;W_T),Y\big)\in[0,1].
\end{align}
Exactly as in the proof of Theorem~\ref{th:itmoe}, we have the Markov chain $S\to W\to L$, hence
$I(S;L)\le I(S;W)$.

Applying the Xu--Raginsky decoupling bound to this same $L$ gives
\begin{align}
\label{eq:xu-bound-2}
\big|\mathbb E[R(W)]-\mathbb E[R_S(W)]\big|
\le \sqrt{\frac{2}{m}\,I(S;L)}
\le \sqrt{\frac{2}{m}\,I(S;W)}.
\end{align}

Plugging \eqref{eq:xu-bound-2} into \eqref{eq:reduce-to-gap} yields
\begin{align}
\mathbb E[R(W)]
\le D(R_g)+\delta_m+\sqrt{\frac{2}{m}\,I(S;W)},   
\end{align}
which is the claim \eqref{eq:RDgen}. \qedhere
\end{proof}

\end{document}